\documentclass{article}

\PassOptionsToPackage{numbers,compress}{natbib}

\usepackage[preprint]{neurips_2026}
\usepackage[utf8]{inputenc}
\usepackage[T1]{fontenc}
\usepackage{hyperref}
\usepackage{url}
\usepackage{booktabs}
\usepackage{nicefrac}
\usepackage{microtype}
\usepackage{xcolor}
\usepackage{amsmath,amssymb,amsfonts,mathtools}
\usepackage{amsthm}
\usepackage{graphicx}
\usepackage{tabularx}
\usepackage{comment}
\usepackage{enumitem}
\usepackage{tikz}
\usepackage{algorithm}

\usepackage{algorithm}
\usepackage{algpseudocode}
\usepackage[nameinlink,noabbrev]{cleveref}

\crefname{assumption}{assumption}{assumptions}
\Crefname{assumption}{Assumption}{Assumptions}
\usepackage{graphicx} 
\usepackage{amsmath}
\usepackage{tabularx}
\usepackage{xcolor}
\usepackage{amsthm}  
\usepackage{comment}
\usepackage{mathabx}
\newtheorem{theorem}{Theorem}

\newtheorem{corollary}{Corollary}

\usepackage{enumitem}

\newcounter{subassumption}[assumption]

\crefname{subassumption}{Assumption}{Assumptions}
\Crefname{subassumption}{Assumption}{Assumptions}

\usepackage{amsthm, amsmath}
\usepackage{tikz}

\title{Matrix-Valued Optimism is Matrix-Valued Augmentation: Additive Hybrid Designs for Constrained Optimization }
\author{
  Jiayi Zhao \\
  University of Washington \\
  Seattle, WA, USA \\
  \texttt{zjiay@uw.edu}
}

\begin{document}

\maketitle
\begin{abstract}
Augmented Lagrangian and optimistic primal--dual methods stabilize equality-constrained optimization through seemingly different mechanisms: the former adds constraint-dependent primal curvature, while the latter adds dual memory. Recent work has shown that these mechanisms are equivalent for scalar parameters. We extend this equivalence to matrix-valued correction. We prove an additivity principle: for symmetric matrix parameters, the ideal primal trajectory depends only on the summed correction matrix, not on how it is split between augmented and optimistic channels. This exposes a design freedom: algebraically equivalent decompositions can have different finite-step feasibility because augmented correction affects primal curvature, whereas optimistic correction affects the scale of the dual memory correction. We formulate the resulting step-size-limited design problem and derive a closed-form hybrid rule that selects a matrix correction, splits it between the two channels, and chooses primal and dual steps using local spectral weights. Experiments on nonlinear equality-constrained problems with controlled constraint-Jacobian conditioning show that the hybrid design improves over pure augmented and pure optimistic endpoints, closely tracks a grid-search hybrid oracle, and is competitive with first-order primal--dual baselines under mild-to-moderate ill-conditioning. The experiments also identify the expected limitation: exact cancellation requires increasingly large matrix corrections as the constraint Jacobian becomes ill-conditioned.
\end{abstract}

\section{Introduction}
\label{sec:introduction}

Equality-constrained optimization is a central primitive in machine learning, control, and scientific computing, where constraints encode feasibility, safety, conservation laws, or structural requirements. First-order primal--dual methods are attractive in these settings because they are simple, scalable, and compatible with differentiable models. However, their dynamics can be highly sensitive to the geometry of the constraint map: primal descent and dual ascent may interact through poorly conditioned constraint directions, leading to oscillations, slow damping, and restrictive step-size choices \citep{arrow1958studies,nocedal2006numerical,chambolle2011first}.

Two widely used mechanisms address this instability. Augmented Lagrangian methods add constraint-dependent curvature to the primal update \citep{hestenes1969multiplier,powell1969method,rockafellar1974augmented,nocedal2006numerical}, while optimistic and extra-gradient-type methods add memory corrections that damp rotational behavior in saddle dynamics \citep{korpelevich1976extragradient,popov1980modification,daskalakis2018training,mokhtari2020unified}. Although these mechanisms are usually viewed as distinct, recent work has shown that, in the scalar-parameter setting, dual optimistic ascent can be algebraically equivalent to a primal-first augmented Lagrangian method with a matching scalar penalty \citep{ramirez2026dualoptimisticascentpi}. This suggests that augmentation and optimism may be two realizations of a common correction mechanism.

The scalar setting, however, is often too restrictive. A scalar penalty \(c\|h(x)\|^2/2\) treats all residual coordinates uniformly, even when different constraint directions have different curvature, scaling, reliability, or computational cost. Nonuniform penalty choices, including one penalty parameter per constraint or multiple augmented-Lagrangian penalties, have been used to address heterogeneous constraint scales \citep{andreani2008augmented,antil2023alesqp}. In such cases, the natural correction object is matrix-valued:
\[
    \frac12 h(x)^\top C h(x).
\]
This raises the main question of the paper: does the equivalence between optimism and augmentation persist for matrix-valued correction, and if so, can this equivalence be used algorithmically?

We show that the answer is yes. For symmetric matrix parameters, matrix-valued optimistic correction and matrix-valued augmented correction are equivalent at the level of primal trajectories. More generally, we prove an additivity principle: if a correction matrix is decomposed as
\[
    M=C_{\rm aug}+C_{\rm opt},
\]
then the ideal primal trajectory depends only on the total matrix \(M\), not on how \(M\) is split between augmented and optimistic channels, provided the effective dual variables are initialized consistently.

This equivalence exposes a design freedom. Although different decompositions can generate the same ideal primal trajectory, they can have very different finite-step behavior. The augmented component \(C_{\rm aug}\) changes the local primal curvature and therefore the admissible primal step size. The optimistic component \(C_{\rm opt}\) changes the scale of the explicit dual memory correction and therefore interacts with primal--dual product restrictions. Thus, the split between augmented and optimistic correction becomes an implementation choice even when the underlying matrix correction is fixed.

We use this observation to formulate a step-size-limited matrix correction design problem and derive a closed-form hybrid rule. The rule constructs a target matrix correction that makes the local augmented curvature positive definite, decomposes this correction between the two channels, and selects primal and dual step sizes using local spectral weights. Experiments on nonlinear equality-constrained problems with controlled constraint-Jacobian conditioning show that the hybrid rule improves over pure augmented and pure optimistic endpoints, closely tracks a grid-search hybrid oracle, and is competitive with first-order primal--dual baselines in the mild-to-moderate ill-conditioned regime. The experiments also identify the expected limitation of exact cancellation: the required matrix correction grows rapidly as the constraint Jacobian becomes highly ill-conditioned.

\paragraph{Contributions.}
Our contributions are:
\begin{itemize}
    \item We extend the scalar equivalence between dual optimistic ascent and augmented Lagrangian dynamics to symmetric matrix-valued correction.
    \item We prove an additivity principle showing that augmented and optimistic matrix correction enter the ideal primal trajectory through their sum.
    \item We formulate a step-size-limited decomposition problem and derive a closed-form hybrid design with a target-curvature guarantee.
    \item We validate the hybrid design on nonlinear equality-constrained problems with controlled constraint-Jacobian conditioning, comparing against pure augmented, pure optimistic, OGDA, extragradient, PDHG-style, and linearized augmented-Lagrangian baselines.
\end{itemize}

\paragraph{Notation.}
Unless otherwise specified, \(\|\cdot\|\) denotes the induced \(2\)-norm, \(A^\dagger\) denotes the Moore--Penrose pseudoinverse, and \(J_h(x)\) denotes the Jacobian of \(h\) at \(x\).
\section{Matrix-Valued Augmentation/ Optimism helps solve ill-conditioned constraints}
\label{sec:why-matrix-correction}

Classical augmented Lagrangian methods typically use a scalar penalty,
\[
    \mathcal L_c(x,\mu)
    =
    f(x)+\mu^\top h(x)+\frac{c}{2}\|h(x)\|^2 ,
\]
which corresponds to the matrix choice \(C=cI\). This isotropic correction is simple, but it treats all residual coordinates uniformly. In many constrained learning and control problems, different residual directions can have different scaling, curvature, reliability, or implementation cost. A matrix-valued penalty,
\[
    \mathcal L_C(x,\mu)
    =
    f(x)+\mu^\top h(x)+\frac12 h(x)^\top C h(x),
\]
allows the correction to be shaped according to the geometry of the constraint residual. This is consistent with augmented-Lagrangian practice: methods with one penalty per constraint already move beyond a single scalar penalty toward diagonal matrix penalties, and augmented-Lagrangian SQP methods often use multiple penalties to handle constraints with disparate scalings.

The advantage is visible in a two-dimensional local model. Consider
\[
    \min_{x\in\mathbb R^2} -\frac12\|x\|^2
    \qquad
    \text{s.t.}
    \qquad
    Bx=0,
    \qquad
    B=
    \begin{bmatrix}
        1000 & 0\\
        0 & 1
    \end{bmatrix}.
\]
The local primal curvature is \(A=-I\). A scalar augmentation \(C=cI\) must take \(c\ge 1\) to make
\[
    A+B^\top C B \succeq 0.
\]
At the smallest feasible scalar value,
\[
    A+B^\top B
    =
    \begin{bmatrix}
        999999 & 0\\
        0 & 0
    \end{bmatrix},
\]
so the weakly scaled constraint direction forces enormous curvature in the strongly scaled direction. By contrast, the matrix choice
\[
    C=
    \begin{bmatrix}
        8.96\cdot 10^{-6} & 0\\
        0 & 8.96
    \end{bmatrix}
\]
gives
\[
    B^\top C B=8.96I,
    \qquad
    A+B^\top C B=7.96I.
\]
Thus a matrix penalty can balance the augmented curvature, while a scalar penalty cannot distinguish the residual scalings. In the corresponding augmented-Lagrangian local Jacobian, this balancing improves the optimized spectral gap by roughly three orders of magnitude; details are given in Appendix~\ref{appapp:stiffness-example}.

This example motivates the matrix-valued viewpoint used throughout the paper. Matrix correction is useful not only as a more flexible augmented penalty, but also as the object that can be redistributed between augmented and optimistic channels. The next section proves that these channels are algebraically additive: the ideal primal trajectory depends on the summed matrix correction, while the implementation can choose how to realize that correction.
\section{Equivalence and Additivity of Matrix-Valued correction}
\label{sec:matrix-correction-equivalence}

We now show that matrix-valued augmented correction and matrix-valued optimistic correction are two algebraic realizations of the same underlying correction. Throughout this section, write
\[
    h_t := h(x_t),
    \qquad
    B_t := J_h(x_t).
\]

\paragraph{Pure matrix-valued augmented update.}
Given a matrix \(C\), the primal-first matrix-augmented update is
\[
    \mu_{t+1}
    =
    \mu_t+\eta_{\rm d}h_t,
\]
\[
    x_{t+1}
    =
    x_t
    -
    \eta_x
    \left[
        \nabla f(x_t)
        +
        B_t^\top
        \left(
            \mu_{t+1}+Ch_t
        \right)
    \right].
    \tag{AUG}
\]
This is gradient descent--ascent on the matrix-augmented Lagrangian
\[
    \mathcal L_C(x,\mu)
    =
    f(x)+\mu^\top h(x)+\frac12 h(x)^\top C h(x).
\]

\paragraph{Pure matrix-valued optimistic update.}
Given a matrix \(\Omega\), the matrix-optimistic dual update is
\[
    \mu_{t+1}
    =
    \mu_t+\eta_{\rm d}h_t+\Omega(h_t-h_{t-1}),
\]
\[
    x_{t+1}
    =
    x_t
    -
    \eta_x
    \left[
        \nabla f(x_t)+B_t^\top\mu_{t+1}
    \right].
    \tag{OPT}
\]
The optimistic term is a residual-memory correction in the dual channel.

\paragraph{Hybrid matrix-valued update.}
The hybrid method combines both mechanisms. Given an augmented correction matrix \(C\) and an optimistic correction matrix \(\Omega\), define \(ALG(C,\Omega)\) by
\[
    \mu_{t+1}
    =
    \mu_t+\eta_{\rm d}h_t+\Omega(h_t-h_{t-1}),
\]
\[
    x_{t+1}
    =
    x_t
    -
    \eta_x
    \left[
        \nabla f(x_t)
        +
        B_t^\top
        \left(
            \mu_{t+1}+Ch_t
        \right)
    \right].
    \tag{HYB}
\]
Thus \(ALG(C,0)\) is the pure augmented method and \(ALG(0,\Omega)\) is the pure optimistic method.

The key observation is that the augmented component \(C\) can be absorbed into the optimistic channel by a change of dual variables. Define the effective dual variable
\[
    \nu_t := \mu_t+C h_{t-1}.
\]
Then \(ALG(C,\Omega)\) is equivalent to the pure optimistic method \(ALG(0,C+\Omega)\), initialized with \(\nu_0=\mu_0+C h_{-1}\). This gives the following equivalence theorem.

\begin{theorem}[Matrix-valued augmented--optimistic equivalence]
\label{thm:matrix-correction-equivalence}
Let \(C,\Omega\in\mathbb R^{m\times m}\), with \(C=C^\top\). Consider \(ALG(C,\Omega)\) and \(ALG(0,C+\Omega)\) with the same primal history,
\[
    x_{-1}^{P}=x_{-1}^{H},
    \qquad
    x_0^{P}=x_0^{H},
\]
and with matched effective-dual initialization,
\[
    \nu_0^P=\mu_0^H+C h(x_{-1}^H).
\]
Then, for every \(t\ge0\),
\[
    x_t^P=x_t^H,
    \qquad
    \nu_t^P=\mu_t^H+C h(x_{t-1}^H).
\]
In particular, the two methods generate identical primal trajectories.
\end{theorem}

\begin{proof}
The proof is by induction. The initialization gives
\[
    x_0^P=x_0^H,
    \qquad
    \nu_0^P=\mu_0^H+C h(x_{-1}^H).
\]
Assume that, for some \(t\ge0\),
\[
    x_t^P=x_t^H,
    \qquad
    x_{t-1}^P=x_{t-1}^H,
    \qquad
    \nu_t^P=\mu_t^H+C h(x_{t-1}^H).
\]
Using the pure optimistic dual update with summed correction \(C+\Omega\),
\[
\begin{aligned}
    \nu_{t+1}^P
    &=
    \nu_t^P
    +
    \eta_{\rm d} h(x_t^P)
    +
    (C+\Omega)
    \left(
        h(x_t^P)-h(x_{t-1}^P)
    \right)
    \\
    &=
    \mu_t^H
    +
    C h(x_{t-1}^H)
    +
    \eta_{\rm d} h(x_t^H)
    +
    (C+\Omega)
    \left(
        h(x_t^H)-h(x_{t-1}^H)
    \right)
    \\
    &=
    \mu_t^H
    +
    \eta_{\rm d}h(x_t^H)
    +
    \Omega
    \left(
        h(x_t^H)-h(x_{t-1}^H)
    \right)
    +
    C h(x_t^H)
    \\
    &=
    \mu_{t+1}^H+C h(x_t^H).
\end{aligned}
\]
This proves the effective-dual relation at time \(t+1\). Substituting this identity into the pure optimistic primal update gives
\[
\begin{aligned}
    x_{t+1}^P
    &=
    x_t^P
    -
    \eta_x
    \left[
        \nabla f(x_t^P)
        +
        B_t^{P\top}\nu_{t+1}^P
    \right]
    \\
    &=
    x_t^H
    -
    \eta_x
    \left[
        \nabla f(x_t^H)
        +
        B_t^{H\top}
        \left(
            \mu_{t+1}^H+C h(x_t^H)
        \right)
    \right]
    \\
    &=
    x_{t+1}^H.
\end{aligned}
\]
Therefore the induction closes, and the two methods generate identical primal trajectories.
\end{proof}

The theorem implies that augmented and optimistic matrix correction enter the ideal primal dynamics through their sum. Therefore, any two decompositions with the same total matrix are equivalent after effective-dual alignment.

\begin{corollary}[Additivity of matrix-valued correction]
\label{cor:additivity}
Let \((C_1,\Omega_1)\) and \((C_2,\Omega_2)\) satisfy
\[
    C_1+\Omega_1=C_2+\Omega_2.
\]
Consider the two hybrid methods \(ALG(C_1,\Omega_1)\) and \(ALG(C_2,\Omega_2)\). If their primal histories and effective dual variables are initialized consistently, then they generate identical primal trajectories:
\[
    x_t^{(1)}=x_t^{(2)}
    \qquad
    \forall t\ge0.
\]
\end{corollary}

\begin{proof}
By \Cref{thm:matrix-correction-equivalence},
\[
    ALG(C_1,\Omega_1)
    \equiv
    ALG(0,C_1+\Omega_1).
\]
Since \(C_1+\Omega_1=C_2+\Omega_2\),
\[
    ALG(0,C_1+\Omega_1)
    =
    ALG(0,C_2+\Omega_2).
\]
Applying \Cref{thm:matrix-correction-equivalence} again gives
\[
    ALG(0,C_2+\Omega_2)
    \equiv
    ALG(C_2,\Omega_2).
\]
Thus
\[
    ALG(C_1,\Omega_1)
    \equiv
    ALG(0,C_1+\Omega_1)
    =
    ALG(0,C_2+\Omega_2)
    \equiv
    ALG(C_2,\Omega_2),
\]
which is precisely equality of the generated primal trajectories under the consistent initialization.
\end{proof}

\paragraph{Implication.}
Additivity identifies an equivalence class of algorithms with the same ideal primal trajectory. The decomposition is nevertheless important for implementation: \(C\) changes the local primal curvature, while \(\Omega\) changes the scale of the explicit dual memory term. The next section uses this distinction to formulate a step-size-limited design problem.
\section{Step-Size-Limited correction Design}
\label{sec:stepsize-design}

The additivity theorem shows that the ideal primal trajectory depends on the summed correction matrix
\[
    M=C_{\rm aug}+C_{\rm opt}.
\]

We study this tradeoff in the local equality-constrained model around a KKT point \((x^\star,\mu^\star)\). Let
\[
    H:=\nabla^2_{xx}\mathcal L(x^\star,\mu^\star),
    \qquad
    B:=J_h(x^\star),
\]
where \(\mathcal L(x,\mu)=f(x)+\mu^\top h(x)\). We assume the local model is augmentable: there exists a symmetric matrix \(M\) such that
\[
    H+B^\top M B \succ 0.
\]
Since \(B^\top M B\) cannot affect directions in \(\ker(B)\), this requires
\[
    z^\top H z>0
    \qquad
    \forall z\in\ker(B)\setminus\{0\}.
\]
Directions visible through \(B\) can then be regularized by matrix correction.

Given a decomposition \(M=C_{\rm aug}+C_{\rm opt}\), the hybrid update has the form
\[
    \mu_{t+1}
    =
    \mu_t+\eta_{\rm d}h_t+C_{\rm opt}(h_t-h_{t-1}),
\]
\[
    x_{t+1}
    =
    x_t-\eta_x
    \left[
        \nabla f(x_t)
        +
        B^\top
        \left(
            \mu_{t+1}+C_{\rm aug}h_t
        \right)
    \right].
\]
The design question is which decomposition of \(M\) permits the largest useful primal and dual steps.

See \cref{app:stepsize-constraints} for constraints on step size (S1 - S3).
 
\section{Closed-Form Hybrid Correction Design}
\label{sec:closed-form-design}

The design problem in \Cref{sec:stepsize-design} can be solved by numerical search, but our goal is an explicit rule. We call a realization \emph{hybrid} when the total correction matrix \(M\) is split between an augmented-channel correction \(C_{\rm aug}\) and an optimistic-channel correction \(C_{\rm opt}\), with
\[
    M=C_{\rm aug}+C_{\rm opt}.
\]
We now derive a closed-form hybrid correction design that satisfies the target-curvature condition, admits an augmented--optimistic decomposition, and chooses primal and dual steps according to the weighted local margin.

At iteration \(t\), let
\[
    H_t:=\nabla^2_{xx}\mathcal L(x_t,\mu_t),
    \qquad
    B_t:=J_h(x_t).
\]
Assume \(B_t\) has full column rank on the local region of interest. Let \(R_t\succ0\) be a chosen residual metric; in the experiments we use \(R_t=I\). Define the cancellation matrix
\[
    C_{{\rm can},t}
    :=
    -B_t^{\dagger\top}H_tB_t^\dagger .
\]
Since \(B_t^\dagger B_t=I\),
\[
    H_t+B_t^\top C_{{\rm can},t}B_t=0.
\]
Thus \(C_{{\rm can},t}\) exactly cancels the local curvature in directions visible through the constraint Jacobian.

To obtain a positive target curvature, choose
\[
    \alpha_t
    =
    \frac{\delta}
    {\lambda_{\min}(B_t^\top R_tB_t)}
\]
and set
\[
    M_t=C_{{\rm can},t}+\alpha_tR_t.
\]
Then
\[
    H_t+B_t^\top M_tB_t
    =
    \alpha_tB_t^\top R_tB_t
    \succeq
    \delta I.
\]
Therefore \(M_t\) is the minimum-strength member of the family \(C_{{\rm can},t}+\alpha R_t\) that achieves the target curvature margin \(\delta\).

We decompose \(M_t\) as
\[
    C_{{\rm aug},t}
    =
    C_{{\rm can},t}
    +
    \theta_t\alpha_tR_t,
    \qquad
    C_{{\rm opt},t}
    =
    (1-\theta_t)\alpha_tR_t,
\]
so that
\[
    C_{{\rm aug},t}+C_{{\rm opt},t}=M_t.
\]
By \Cref{cor:additivity}, all choices of \(\theta_t\in[0,1]\) realize the same ideal total correction \(M_t\). The role of \(\theta_t\) is therefore not to change the ideal primal trajectory, but to choose a more implementable realization.

Under the cancellation family,
\[
    H_t+B_t^\top C_{{\rm aug},t}B_t
    =
    \theta_t\alpha_tB_t^\top R_tB_t,
\]
while
\[
    \|C_{{\rm opt},t}\|_2
    =
    (1-\theta_t)\alpha_t\|R_t\|_2.
\]
Balancing the primal curvature restriction \((S1)\) with the optimism-induced restriction obtained from \((S2)\)--\((S3)\) gives
\[
    \theta_t
    =
    \frac{
        \kappa_x\|R_t\|_2\|B_t\|_2^2
    }{
        \kappa_x\|R_t\|_2\|B_t\|_2^2
        +
        \kappa_p\kappa_\omega\|B_t^\top R_tB_t\|_2
    }.
\]
This rule balances the cost of placing correction in the augmented channel against the cost of placing correction in the optimistic channel.

Finally, define
\[
    A_{M_t}=H_t+B_t^\top M_tB_t.
\]
Following the weighted local-margin model from \Cref{sec:stepsize-design}, estimate
\[
    a_x(t)=\frac12\lambda_{\min}(A_{M_t}),
    \qquad
    a_{\rm d}(t)
    =
    \lambda_{\min}^+
    \left(
        B_tA_{M_t}^{-1}B_t^\top
    \right).
\]
The step sizes are chosen by maximizing
\[
    \min\{a_x(t)\eta_x,a_{\rm d}(t)\eta_{\rm d}\}
\]
subject to \((S1)\)--\((S3)\). This gives
\[
    \eta_x
    =
    \min\left\{
        \sqrt{\frac{a_{\rm d}(t)\kappa_p}{a_x(t)\|B_t\|_2^2}},
        \frac{\kappa_x}{\|H_t+B_t^\top C_{{\rm aug},t}B_t\|_2},
        \frac{\kappa_p\kappa_\omega}
        {\|C_{{\rm opt},t}\|_2\|B_t\|_2^2}
    \right\},
\]
where the last term is omitted if \(C_{{\rm opt},t}=0\), and
\[
    \eta_{\rm d}
    =
    \frac{\kappa_p}{\eta_x\|B_t\|_2^2}.
\]
The short derivation of this step rule is given in Appendix~\ref{app:closed-form-steps}.

\begin{algorithm}[t]
\caption{Closed-form hybrid correction design}
\label{alg:closed-hybrid}
\begin{algorithmic}
\Require \(x_t,\mu_t,h_{t-1},C_{{\rm opt},t-1}\), constants \(\delta,\kappa_x,\kappa_p,\kappa_\omega\), metric \(R_t\succ0\)
\State Compute \(H_t=\nabla^2_{xx}\mathcal L(x_t,\mu_t)\) and \(B_t=J_h(x_t)\)
\State \(C_{{\rm can},t}\gets -B_t^{\dagger\top}H_tB_t^\dagger\)
\State \(\alpha_t\gets \delta/\lambda_{\min}(B_t^\top R_tB_t)\), \(M_t\gets C_{{\rm can},t}+\alpha_tR_t\)
\State Compute \(\theta_t\) by the balancing rule
\State \(C_{{\rm aug},t}\gets C_{{\rm can},t}+\theta_t\alpha_tR_t\), \(C_{{\rm opt},t}\gets(1-\theta_t)\alpha_tR_t\)
\State Compute \(a_x(t),a_{\rm d}(t)\), then choose \(\eta_x,\eta_{\rm d}\)
\State \(h_t\gets h(x_t)\)
\State \(\mu_{t+1}\gets \mu_t+\eta_{\rm d}h_t+C_{{\rm opt},t}h_t-C_{{\rm opt},t-1}h_{t-1}\)
\State \(x_{t+1}\gets x_t-\eta_x[\nabla f(x_t)+B_t^\top(\mu_{t+1}+C_{{\rm aug},t}h_t)]\)
\end{algorithmic}
\end{algorithm}

The construction has the following immediate guarantee; proof in Appendix~\ref{app:closed-form-proofs}.

\begin{theorem}[Target curvature and additive realization]
\label{thm:closed-form-feasibility}
Assume \(B_t\) has full column rank and \(R_t\succ0\). Then
\[
    H_t+B_t^\top M_tB_t\succeq \delta I,
    \qquad
    C_{{\rm aug},t}+C_{{\rm opt},t}=M_t.
\]
Consequently, by \Cref{cor:additivity}, the closed-form hybrid realization has the same ideal primal trajectory as the pure realization using the full target correction \(M_t\), under matched effective-dual initialization.
\end{theorem}
\section{Numerical Experiments}
\label{sec:experiments}

We evaluate whether the proposed hybrid correction design realizes the tradeoff predicted by the theory. The experiments are designed to test three questions: whether hybrid decompositions improve over the pure augmented and pure optimistic endpoints; whether the closed-form rule tracks a grid-searched hybrid oracle; and how the design behaves as the constraint Jacobian becomes increasingly ill-conditioned.

\paragraph{Problem class.}
We generate nonlinear equality-constrained problems of the form
\[
    \min_{x\in\mathbb R^3} f(x)
    \qquad
    \text{s.t.}
    \qquad
    h(x)=0,
    \qquad
    h:\mathbb R^3\to\mathbb R^5 .
\]
Each instance is constructed with a known KKT point \((x^\star,\mu^\star)\). The objective is nonlinear and nonconvex, consisting of an indefinite quadratic term together with sinusoidal and quartic nonlinear terms. The constraint is also nonlinear, but its Jacobian at the solution is controlled explicitly. Specifically, we generate
\[
    B_\star := J_h(x^\star)=U\Sigma V^\top,
\]
where \(U\in\mathbb R^{5\times 3}\) and \(V\in\mathbb R^{3\times 3}\) have orthonormal columns, and the singular values in \(\Sigma\) are chosen so that
\[
    \operatorname{cond}(B_\star)
    \in
    \{1,2,3,5,8,13\}.
\]
For each condition level, we generate five random problem instances, for a total of \(30\) nonlinear equality-constrained problems. The nonlinear part of \(h\) is chosen so that \(h(x^\star)=0\) and \(J_h(x^\star)=B_\star\), while the objective is shifted so that
\[
    \nabla f(x^\star)+J_h(x^\star)^\top\mu^\star=0.
\]
Thus each instance has a known KKT point, and we can directly measure convergence using \(\|x_t-x^\star\|\). Further details of the generator are given in Appendix~\ref{app:experiment-details}.

\paragraph{Methods.}
We compare the proposed \emph{closed-form hybrid correction} against endpoint and first-order primal--dual baselines. The matrix-design baselines are: pure augmented, which sets \(C_{\rm aug}=M\) and \(C_{\rm opt}=0\); pure optimistic, which sets \(C_{\rm aug}=0\) and \(C_{\rm opt}=M\); and a grid-searched hybrid oracle, which searches over the same cancellation family as the proposed method but chooses the residual strength and split parameter by grid search. For each matrix-design method, we report both a unit-margin variant, using \(a_x=a_{\rm d}=1\), and a weighted variant using the spectral constants from \Cref{sec:stepsize-design}.

We also compare against standard step-based first-order baselines: OGDA/optimistic gradient applied to the Lagrangian saddle operator \citep{popov1980modification,daskalakis2018training}; extragradient applied to the same saddle operator \citep{korpelevich1976extragradient}; a PDHG-style primal--dual update motivated by the Chambolle--Pock method \citep{chambolle2011first}; and a linearized augmented-Lagrangian baseline, motivated by the classical method-of-multipliers and augmented-Lagrangian literature \citep{hestenes1969multiplier,powell1969method,rockafellar1974augmented,nocedal2006numerical}. Since the constraints are nonlinear, the PDHG-style and linearized augmented-Lagrangian methods should be interpreted as first-order nonlinear analogues rather than exact convex-composite PDHG or exact ADMM.

\paragraph{Metrics.}
For each method and instance, we record the first hitting time for
\[
    \|x_t-x^\star\|\le 10^{-2},
    \qquad
    \|x_t-x^\star\|\le 10^{-3},
    \qquad
    \|x_t-x^\star\|\le 10^{-4}.
\]
A method is counted as unsuccessful at a given threshold if it does not reach the threshold within the iteration budget or if the matrix-design step becomes infeasible under the imposed norm budget. We also record feasibility \(\|h(x_t)\|\), stationarity \(\|\nabla f(x_t)+J_h(x_t)^\top\mu_t\|\), and, for matrix-design methods, the maximum correction norm \(\max_t\|M_t\|_2\).

\paragraph{Success rates.}
\Cref{tab:illcond-success} reports success counts at the threshold \(\|x_t-x^\star\|\le 10^{-3}\). The weighted hybrid methods are reliable for mild-to-moderate ill-conditioning. In particular, the grid-searched weighted hybrid succeeds on all five instances for condition levels \(1,2,3,\) and \(5\), while the closed-form weighted hybrid succeeds on all five instances for levels \(2,3,\) and \(5\), and on four out of five instances at level \(1\). By contrast, pure augmented correction becomes unreliable once the constraint geometry is mildly ill-conditioned, and pure optimistic correction is less robust at higher condition levels. OGDA and extragradient do not reach the target accuracy in these experiments, while the PDHG-style baseline is strong at low condition numbers but also degrades as the condition number increases.

\begin{table}[t]
\centering
\caption{Success counts at \(\|x_t-x^\star\|\le 10^{-3}\), out of five random instances per condition level. The hybrid methods are reliable under mild-to-moderate ill-conditioning and degrade when the required correction matrix becomes too large.}
\label{tab:illcond-success}
\begin{tabular}{c|cccccccc}
\toprule
\(\operatorname{cond}(B_\star)\)
&
Pure Aug. W
&
Pure Opt. W
&
Grid Hyb. W
&
Closed Hyb. W
&
OGDA
&
EG
&
PDHG
&
Lin. AL
\\
\midrule
1  & 5 & 5 & 5 & 4 & 0 & 0 & 5 & 0 \\
2  & 0 & 5 & 5 & 5 & 0 & 0 & 5 & 0 \\
3  & 2 & 5 & 5 & 5 & 0 & 0 & 5 & 0 \\
5  & 0 & 4 & 5 & 5 & 0 & 0 & 2 & 0 \\
8  & 0 & 0 & 1 & 0 & 0 & 0 & 1 & 0 \\
13 & 0 & 0 & 0 & 0 & 0 & 0 & 1 & 0 \\
\bottomrule
\end{tabular}
\end{table}

\paragraph{Iteration counts.}
Among successful runs, the weighted hybrid methods are fast and closely track the grid-searched hybrid oracle. \Cref{tab:illcond-hits} reports median hitting times to reach \(\|x_t-x^\star\|\le 10^{-3}\). At condition level \(1\), the PDHG-style baseline is fastest. However, as the condition number increases to \(3\) and \(5\), the weighted hybrid methods become competitive with PDHG while remaining substantially more reliable than the pure augmented endpoint. The closed-form hybrid is consistently close to the grid-searched hybrid oracle in the regime where the matrix design remains feasible.

\begin{table}[t]
\centering
\caption{Median hitting time to reach \(\|x_t-x^\star\|\le 10^{-3}\). A dash means fewer than half of the runs reached the threshold. The closed-form weighted hybrid closely tracks the grid-searched hybrid oracle in the mild-to-moderate ill-conditioned regime.}
\label{tab:illcond-hits}
\begin{tabular}{c|cccc}
\toprule
\(\operatorname{cond}(B_\star)\)
&
Grid Hyb. W
&
Closed Hyb. W
&
Closed Hyb. Unit
&
PDHG
\\
\midrule
1  & 15  & 19.5 & 69 & 4 \\
2  & 35  & 36   & 44 & 22 \\
3  & 43  & 49   & 55 & 44 \\
5  & 60  & 71   & 75 & 74 \\
8  & 60  & 112  & 86 & 61 \\
13 & 141 & 160  & -- & 150 \\
\bottomrule
\end{tabular}
\end{table}

\paragraph{Effect of spectral weighting.}
The weighted variants generally improve over the unit-margin variants. This agrees with the local analysis in \Cref{sec:stepsize-design}: the contraction proxy is not \(\min\{\eta_x,\eta_{\rm d}\}\), but rather a weighted margin
\[
    \min\{a_x\eta_x,a_{\rm d}\eta_{\rm d}\},
\]
where \(a_x\) and \(a_{\rm d}\) depend on the local augmented curvature \(A_M\) and the constraint Jacobian \(B\). Empirically, using these constants improves hitting times, especially for well-conditioned and mildly ill-conditioned instances.

\paragraph{Failure under ill-conditioning.}
The degradation at large condition numbers is expected. The closed-form design uses the cancellation matrix
\[
    C_{{\rm can},t}
    =
    -B_t^{\dagger\top}H_tB_t^\dagger .
\]
Hence its norm can scale like
\[
    \|C_{{\rm can},t}\|_2
    \lesssim
    \|H_t\|_2\|B_t^\dagger\|_2^2.
\]
As \(\sigma_{\min}(B_t)\) decreases, the pseudoinverse amplifies the correction. \Cref{tab:Mnorm-growth} reports representative median values of \(\max_t\|M_t\|_2\), showing rapid growth as the constraint Jacobian becomes more ill-conditioned. This explains why the exact cancellation design begins to fail around condition levels \(8\) and \(13\): the required correction exceeds the imposed norm budget.

\begin{table}[t]
\centering
\caption{Representative growth of \(\max_t\|M_t\|_2\) with the condition number of the constraint Jacobian. The trend matches the pseudoinverse scaling of \(C_{{\rm can},t}\).}
\label{tab:Mnorm-growth}
\begin{tabular}{c|cc}
\toprule
\(\operatorname{cond}(B_\star)\)
&
Closed Hyb. Unit
&
Grid Hyb. W
\\
\midrule
1  & 1.98  & 2.45 \\
2  & 4.74  & 6.37 \\
3  & 8.94  & 21.54 \\
5  & 21.03 & 48.81 \\
8  & 81.16 & 92.06 \\
13 & --    & 299.94 \\
\bottomrule
\end{tabular}
\end{table}

Overall, the experiments support the predicted tradeoff. When the matrix correction is feasible, the hybrid decomposition improves over pure augmented and pure optimistic endpoints and the closed-form rule approximates the grid-searched hybrid oracle. At the same time, the ill-conditioned cases expose a limitation of exact cancellation: the method is a local matrix-correction design, not a globally robust constrained solver. A natural extension is to use partial cancellation,
\[
    C_{{\rm can},t}^{(\gamma)}
    =
    -\gamma_t B_t^{\dagger\top}H_tB_t^\dagger,
    \qquad
    0\le \gamma_t\le 1,
\]
or a trust-region activation rule when the required correction norm becomes too large.
\section{Conclusion}
\label{sec:conclusion}

We studied the relationship between matrix-valued augmented Lagrangian corrections and matrix-valued optimistic dual corrections for equality-constrained optimization. Our main result is an augmented--optimistic equivalence showing that, under an effective-dual initialization, the exact-oracle primal trajectory depends on the summed correction matrix
\[
    M=C_{\rm aug}+C_{\rm opt},
\]
rather than on the particular decomposition between the two channels. This additivity principle creates a design freedom: algebraically equivalent decompositions can have different finite-step behavior because the augmented component changes local primal curvature, while the optimistic component changes the scale of the explicit dual memory correction.

Using this observation, we formulated a step-size-limited correction design problem and derived a closed-form hybrid rule. The rule constructs a local target correction, splits it between augmented and optimistic channels, and selects primal and dual steps using local spectral weights. Experiments on nonlinear equality-constrained problems with controlled constraint-Jacobian conditioning show that the hybrid design improves over pure augmented and pure optimistic endpoints, closely tracks a grid-searched hybrid oracle, and is competitive with first-order primal--dual baselines in the mild-to-moderate ill-conditioned regime.

The results also clarify the limitation of exact cancellation. The construction uses
\[
    C_{{\rm can},t}
    =
    -B_t^{\dagger\top}H_tB_t^\dagger,
\]
so the required correction can grow rapidly when \(B_t\) is ill-conditioned. Thus the proposed method should be understood as a local matrix-correction design rather than a globally robust constrained solver. Future work includes partial-cancellation or trust-region variants that control \(\|M_t\|\), extensions to rank-deficient and inequality-constrained settings, and scalable Hessian or Gauss--Newton approximations for larger learning and control problems.

\newpage
\bibliographystyle{plainnat}
\bibliography{references}
\newpage
\appendix

\section{Additional Related Work}
\label{app:related-work}

\paragraph{Augmented Lagrangian and matrix-valued penalties.}
Augmented Lagrangian and method-of-multipliers methods stabilize constrained optimization by combining multiplier updates with constraint-dependent regularization \citep{hestenes1969multiplier,powell1969method,rockafellar1974augmented,nocedal2006numerical}. Classical presentations often use a scalar penalty parameter, but nonuniform penalty choices are common in practice and theory. For example, augmented Lagrangian methods with one penalty parameter per constraint already correspond to diagonal matrix-valued penalties \citep{andreani2008augmented}. More recent augmented-Lagrangian SQP methods also use multiple penalty parameters to address constraints with different scalings \citep{antil2023alesqp}. Our work differs from these penalty-selection approaches by studying an algebraic equivalence between matrix-valued augmentation and matrix-valued optimistic dual correction, and by using this equivalence to redistribute a total correction across two algorithmic channels.

\paragraph{Optimistic and extra-gradient methods.}
Extra-gradient and optimistic methods are classical approaches for saddle-point and variational-inequality problems \citep{korpelevich1976extragradient,popov1980modification}. They have also become important in machine learning, especially for min--max optimization and GAN training \citep{daskalakis2018training}. Recent analyses connect extra-gradient and optimistic-gradient methods to approximations of proximal-point dynamics \citep{mokhtari2020unified}. Our optimistic channel is related in spirit to these memory and correction mechanisms, but our focus is different: we identify when a matrix-valued optimistic correction is algebraically equivalent to a matrix-valued augmented-Lagrangian correction in equality-constrained optimization.

\paragraph{Primal--dual splitting and step-size restrictions.}
Primal--dual splitting methods such as PDHG impose product-form restrictions on primal and dual step sizes, typically involving the squared norm of a coupling operator \citep{chambolle2011first}. Our local design model uses the analogous restriction \(\eta_x\eta_{\rm d}\|B\|_2^2\le \kappa_p\), where \(B\) is the local constraint Jacobian. This connection motivates the product constraint used in the step-size-limited design problem, but our setting is nonlinear and equality-constrained, so the PDHG-style baseline in the experiments should be interpreted as a first-order nonlinear analogue rather than a direct application of the convex-composite theory.

\paragraph{Scalar augmented--optimistic equivalence.}
The closest starting point to our work is the scalar equivalence between dual optimistic ascent and primal-first augmented Lagrangian dynamics \citep{ramirez2026dualoptimisticascentpi}. We extend this idea in two directions. First, we replace scalar coefficients by symmetric matrix-valued corrections. Second, we prove an additivity principle: a total correction matrix \(M\) can be decomposed as \(M=C_{\rm aug}+C_{\rm opt}\), and the exact-oracle primal trajectory depends only on \(M\), not on the particular augmented--optimistic split. This additivity is what enables the step-size-limited hybrid correction design studied in the main text.

\section{Stiffness Example: Scalar versus Matrix-Valued Augmentation}
\label{app:stiffness-example}

This appendix gives the calculation behind the motivating example in \Cref{sec:why-matrix-feedback}. The goal is to compare how a scalar augmentation and a matrix-valued augmentation affect the local augmented-Lagrangian dynamics.

Consider the quadratic equality-constrained problem
\[
    \min_{x\in\mathbb R^2}
    -\frac12\|x\|^2
    \qquad
    \mathrm{s.t.}
    \qquad
    Bx=0,
    \qquad
    B=
    \begin{bmatrix}
        1000 & 0\\
        0 & 1
    \end{bmatrix}.
\]
The local primal Hessian is
\[
    A=-I.
\]
For an augmented matrix \(C\succeq 0\), the local augmented primal curvature is
\[
    A_C:=A+B^\top C B.
\]
The primal-first augmented Lagrangian linearization around the KKT point has Jacobian
\[
    J_{\rm AL}(\eta_x,\eta_{\rm d};C)
    =
    \begin{bmatrix}
        I-\eta_x A_C
        &
        -\eta_x B^\top
        \\
        \eta_{\rm d}B(I-\eta_x A_C)
        &
        I-\eta_x\eta_{\rm d}BB^\top
    \end{bmatrix}.
\]
For a given augmentation \(C\), we choose the best stable step sizes over a finite grid and measure the local contraction gap
\[
    \mathrm{gap}(C)
    :=
    1-\rho\!\left(J_{\rm AL}(\eta_x,\eta_{\rm d};C)\right),
\]
where \(\rho(\cdot)\) denotes spectral radius. Larger gap means faster local linear convergence.

\paragraph{Scalar augmentation.}
A scalar augmentation has the form
\[
    C=cI.
\]
Then
\[
    A_C
    =
    -I+B^\top cIB
    =
    \begin{bmatrix}
        -1+10^6c & 0\\
        0 & -1+c
    \end{bmatrix}.
\]
To make \(A_C\succeq 0\), the scalar penalty must satisfy
\[
    c\ge 1.
\]
At the smallest feasible scalar penalty \(c=1\),
\[
    A_C
    =
    \begin{bmatrix}
        999999 & 0\\
        0 & 0
    \end{bmatrix}.
\]
Thus the weakly scaled constraint direction forces the scalar penalty to be large, which creates extremely large curvature in the strongly scaled direction. This is the stiffness induced by scalar augmentation.

\paragraph{Matrix-valued augmentation.}
A diagonal matrix-valued augmentation can scale the two residual directions separately. Choose
\[
    C_{\rm mat}
    =
    \begin{bmatrix}
        8.96\cdot 10^{-6} & 0\\
        0 & 8.96
    \end{bmatrix}.
\]
Then
\[
    B^\top C_{\rm mat}B
    =
    \begin{bmatrix}
        1000 & 0\\
        0 & 1
    \end{bmatrix}
    \begin{bmatrix}
        8.96\cdot 10^{-6} & 0\\
        0 & 8.96
    \end{bmatrix}
    \begin{bmatrix}
        1000 & 0\\
        0 & 1
    \end{bmatrix}
    =
    8.96I.
\]
Therefore
\[
    A_{C_{\rm mat}}
    =
    -I+8.96I
    =
    7.96I.
\]
Unlike the scalar penalty, the matrix-valued augmentation balances the augmented curvature across the two coordinates.

\paragraph{Spectral comparison.}
Using the local Jacobian \(J_{\rm AL}\), we compare the best achievable spectral gap under scalar and matrix-valued augmentation. For the scalar family \(C=cI\), the best gap over the grid is approximately
\[
    \max_{c,\eta_x,\eta_{\rm d}}
    \left[
        1-\rho\!\left(J_{\rm AL}(\eta_x,\eta_{\rm d};cI)\right)
    \right]
    \approx
    10^{-6}.
\]
For the matrix-valued choice above, the optimized gap is approximately
\[
    \max_{\eta_x,\eta_{\rm d}}
    \left[
        1-\rho\!\left(J_{\rm AL}(\eta_x,\eta_{\rm d};C_{\rm mat})\right)
    \right]
    \approx
    2\cdot 10^{-3}.
\]
Thus the matrix-valued augmentation improves the optimized local spectral gap by roughly three orders of magnitude.

\paragraph{Interpretation.}
The scalar penalty must satisfy the most restrictive residual direction and therefore over-regularizes directions whose constraint scaling is already large. Matrix-valued augmentation avoids this coupling by assigning different penalty strengths to different residual directions. In this example, the matrix choice makes the augmented curvature nearly isotropic, whereas the smallest feasible scalar penalty produces a nearly singular and extremely ill-conditioned augmented curvature matrix. This illustrates why matrix-valued correction is a natural object for anisotropic constrained problems.

\section{Additive Equivalence Theorems and Proofs}
\begin{theorem}[Hybrid--Aug-Lag equivalence under partial residual-memory noise]
\label{thm:hybrid-aug-partial-noise}
Let $f:\mathbb{R}^d\to\mathbb{R}$ and
$h:\mathbb{R}^d\to\mathbb{R}^m$ be differentiable, and let
\[
J_h(x):=\nabla h(x)\in\mathbb{R}^{m\times d}.
\]
Let $M,C_1,C_2\in\mathbb{R}^{m\times m}$ satisfy
\[
M=C_1+C_2.
\]
Fix step sizes $\eta_x>0$ and $\eta_{\mathrm d}>0$.

Let the residual coordinates be partitioned into noisy and clean coordinates.
Let $P_n$ denote the coordinate projection onto the noisy residual coordinates.
Assume the residual-memory noise satisfies
\[
\widehat h_t
=
h(x_t)+\varepsilon_t,
\qquad
\varepsilon_t=P_n\varepsilon_t,
\]
and assume that the optimistic clean-channel matrix $C_2$ annihilates the noisy
coordinates:
\[
C_2P_n=0.
\]
Equivalently,
\[
C_2\varepsilon_t=0
\qquad
\text{for every admissible residual-memory noise } \varepsilon_t.
\]

Consider the full augmented method
\begin{align}
\lambda_{t+1}^{A}
&=
\lambda_t^{A}
+\eta_{\mathrm d}\,\widehat h_t^{A},
\label{eq:aug-noisy-dual}
\\
x_{t+1}^{A}
&=
x_t^{A}
-\eta_x\Big(
\nabla f(x_t^{A})
+
J_h(x_t^{A})^\top
\big(
\lambda_{t+1}^{A}
+
M h(x_t^{A})
\big)
\Big),
\label{eq:aug-noisy-primal}
\end{align}
where
\[
\widehat h_t^{A}=h(x_t^{A})+\varepsilon_t.
\]
Here the dual update uses the noisy residual-memory value
$\widehat h_t^A$, but the augmented primal correction
$M h(x_t^A)$ uses the exact current residual $h(x_t^A)$.

Consider also the hybrid method
\begin{align}
\mu_{t+1}^{H}
&=
\mu_t^{H}
+\eta_{\mathrm d}\,\widehat h_t^{H}
+
C_2\Big(
\widehat h_t^{H}-\widehat h_{t-1}^{H}
\Big),
\label{eq:hybrid-partial-noisy-dual}
\\
x_{t+1}^{H}
&=
x_t^{H}
-\eta_x\Big(
\nabla f(x_t^{H})
+
J_h(x_t^{H})^\top
\big(
\mu_{t+1}^{H}
+
C_1 h(x_t^{H})
\big)
\Big),
\label{eq:hybrid-partial-noisy-primal}
\end{align}
where
\[
\widehat h_t^{H}=h(x_t^{H})+\varepsilon_t.
\]
Again, the optimistic memory term uses the noisy stored residuals
$\widehat h_t^H,\widehat h_{t-1}^H$, while the augmented correction
$C_1 h(x_t^H)$ uses the exact current residual.

Assume the initial conditions satisfy
\[
x_{-1}^{A}=x_{-1}^{H},
\qquad
x_0^{A}=x_0^{H},
\]
and
\[
\mu_0^{H}
=
\lambda_0^{A}
+
C_2 h(x_{-1}^{A}).
\]
Then, for every $t\ge 0$,
\begin{align}
x_t^{H}&=x_t^{A},
\label{eq:hybrid-aug-x-eq}
\\
\mu_t^{H}
&=
\lambda_t^{A}
+
C_2 h(x_{t-1}^{A}).
\label{eq:hybrid-aug-dual-relation}
\end{align}
In particular, the hybrid method with $(C_1,C_2)$ generates exactly the same
primal trajectory as the full augmented method with $M=C_1+C_2$, for every
realization of the residual-memory noise satisfying $C_2\varepsilon_t=0$.
\end{theorem}
\begin{proof}
We prove \eqref{eq:hybrid-aug-x-eq} and
\eqref{eq:hybrid-aug-dual-relation} simultaneously by induction on $t\ge 0$.

\medskip
\noindent
\textbf{Base step.}
By assumption,
\[
x_{-1}^{A}=x_{-1}^{H},
\qquad
x_0^{A}=x_0^{H}.
\]
Also, by the initialization condition,
\[
\mu_0^{H}
=
\lambda_0^{A}
+
C_2 h(x_{-1}^{A}),
\]
so \eqref{eq:hybrid-aug-dual-relation} holds at $t=0$.

\medskip
\noindent
\textbf{Induction hypothesis.}
Assume that, for some $t\ge 0$,
\begin{align}
x_{t-1}^{H}&=x_{t-1}^{A},
\label{eq:partial-IH-prev}
\\
x_t^{H}&=x_t^{A},
\label{eq:partial-IH-cur}
\\
\mu_t^{H}
&=
\lambda_t^{A}
+
C_2 h(x_{t-1}^{A}).
\label{eq:partial-IH-dual}
\end{align}
We prove the corresponding identities at time $t+1$.

\medskip
\noindent
\textbf{Step 1: dual relation at time $t+1$.}
From the hybrid dual update,
\begin{align*}
\mu_{t+1}^{H}
&=
\mu_t^{H}
+\eta_{\mathrm d}\,\widehat h_t^{H}
+
C_2\left(
\widehat h_t^{H}-\widehat h_{t-1}^{H}
\right).
\end{align*}
Using the induction hypothesis and the fact that both methods use the same
noise realization, we have
\[
\widehat h_t^{H}
=
h(x_t^{H})+\varepsilon_t
=
h(x_t^{A})+\varepsilon_t
=
\widehat h_t^{A},
\]
and similarly
\[
\widehat h_{t-1}^{H}
=
\widehat h_{t-1}^{A}.
\]
Therefore,
\begin{align*}
\mu_{t+1}^{H}
&=
\lambda_t^{A}
+
C_2 h(x_{t-1}^{A})
+
\eta_{\mathrm d}\,\widehat h_t^{A}
+
C_2\left(
\widehat h_t^{A}-\widehat h_{t-1}^{A}
\right).
\end{align*}
Since
\[
\widehat h_s^{A}=h(x_s^{A})+\varepsilon_s
\]
and $C_2\varepsilon_s=0$, we have
\[
C_2\widehat h_s^{A}
=
C_2 h(x_s^{A})
\qquad
\text{for }s=t,t-1.
\]
Hence
\begin{align*}
C_2 h(x_{t-1}^{A})
+
C_2\left(
\widehat h_t^{A}-\widehat h_{t-1}^{A}
\right)
&=
C_2 h(x_{t-1}^{A})
+
C_2 h(x_t^{A})
-
C_2 h(x_{t-1}^{A})
\\
&=
C_2 h(x_t^{A}).
\end{align*}
Thus
\begin{align*}
\mu_{t+1}^{H}
&=
\lambda_t^{A}
+
\eta_{\mathrm d}\,\widehat h_t^{A}
+
C_2 h(x_t^{A}).
\end{align*}
By the full augmented dual update
\[
\lambda_{t+1}^{A}
=
\lambda_t^{A}
+
\eta_{\mathrm d}\,\widehat h_t^{A},
\]
we obtain
\[
\mu_{t+1}^{H}
=
\lambda_{t+1}^{A}
+
C_2 h(x_t^{A}).
\]
This is exactly \eqref{eq:hybrid-aug-dual-relation} at time $t+1$.

\medskip
\noindent
\textbf{Step 2: primal equality at time $t+1$.}
Using the hybrid primal update, the induction hypothesis, and the dual relation
just proved,
\begin{align*}
x_{t+1}^{H}
&=
x_t^{H}
-\eta_x\Big(
\nabla f(x_t^{H})
+
J_h(x_t^{H})^\top
\big(
\mu_{t+1}^{H}
+
C_1 h(x_t^{H})
\big)
\Big)
\\
&=
x_t^{A}
-\eta_x\Big(
\nabla f(x_t^{A})
+
J_h(x_t^{A})^\top
\big(
\lambda_{t+1}^{A}
+
C_2 h(x_t^{A})
+
C_1 h(x_t^{A})
\big)
\Big).
\end{align*}
Since $M=C_1+C_2$, this becomes
\[
x_{t+1}^{H}
=
x_t^{A}
-\eta_x\Big(
\nabla f(x_t^{A})
+
J_h(x_t^{A})^\top
\big(
\lambda_{t+1}^{A}
+
M h(x_t^{A})
\big)
\Big).
\]
By the full augmented primal update \eqref{eq:aug-noisy-primal}, the right-hand
side is exactly $x_{t+1}^{A}$. Therefore,
\[
x_{t+1}^{H}=x_{t+1}^{A}.
\]

\medskip
\noindent
\textbf{Conclusion.}
By induction, \eqref{eq:hybrid-aug-x-eq} and
\eqref{eq:hybrid-aug-dual-relation} hold for all $t\ge 0$.
Thus the hybrid method and the full augmented method generate identical primal
trajectories under the stated partial residual-memory noise model.
\end{proof}
\subsection{Additive property of augmentation and optimism}
\begin{theorem}[Additive property of augmentation and optimism]
\label{thm:additive-property-dual-first}
Let $f:\mathbb{R}^d\to\mathbb{R}$ and $h:\mathbb{R}^d\to\mathbb{R}^m$ be differentiable, and let
\[
J_h(x):=\nabla h(x)\in\mathbb{R}^{m\times d}.
\]
Let $C,\Omega\in\mathbb{R}^{m\times m}$, and assume $C=C^\top$.
Fix step sizes $\eta_x>0$ and $\eta_{\mathrm d}>0$.

Consider the \emph{hybrid augmented--optimistic method}
\begin{align}
\mu_{t+1}^{H}
&=
\mu_t^{H}
+\eta_{\mathrm d}\,h(x_t^{H})
+\Omega\Big(h(x_t^{H})-h(x_{t-1}^{H})\Big),
\label{eq:hybrid-dual-df}
\\
x_{t+1}^{H}
&=
x_t^{H}
-\eta_x\Big(
\nabla f(x_t^{H})
+
J_h(x_t^{H})^\top\big(\mu_{t+1}^{H}+C\,h(x_t^{H})\big)
\Big).
\label{eq:hybrid-primal-df}
\end{align}

Consider also the \emph{pure optimistic method with summed matrix}
\begin{align}
\nu_{t+1}^{P}
&=
\nu_t^{P}
+\eta_{\mathrm d}\,h(x_t^{P})
+(C+\Omega)\Big(h(x_t^{P})-h(x_{t-1}^{P})\Big),
\label{eq:pure-dual-df}
\\
x_{t+1}^{P}
&=
x_t^{P}
-\eta_x\Big(
\nabla f(x_t^{P})
+
J_h(x_t^{P})^\top \nu_{t+1}^{P}
\Big).
\label{eq:pure-primal-df}
\end{align}

Assume the initial conditions satisfy
\[
x_{-1}^{P}=x_{-1}^{H},
\qquad
x_0^{P}=x_0^{H},
\qquad
\nu_0^{P}=\mu_0^{H}+C\,h(x_{-1}^{H}).
\]

Then, for every $t\ge 0$,
\begin{align}
x_t^{P}&=x_t^{H},
\label{eq:xt-equality-df}
\\
\nu_t^{P}&=\mu_t^{H}+C\,h(x_{t-1}^{H}).
\label{eq:dual-relation-df}
\end{align}
In particular, the two methods generate exactly the same primal trajectory.

If one starts from a genuine scalar augmented term
\[
\frac12\,h(x)^\top C\,h(x),
\]
then the assumption $C=C^\top$ is the natural one. If $C$ is not symmetric, then the primal correction induced by that scalar quadratic is obtained by replacing $C$ with
\[
\operatorname{sym}(C):=\frac{C+C^\top}{2}.
\]
\end{theorem}

\begin{proof}
We prove \eqref{eq:xt-equality-df} and \eqref{eq:dual-relation-df} simultaneously by induction on $t\ge 0$.

\medskip
\noindent
\textbf{Base step.}
By assumption,
\[
x_0^{P}=x_0^{H}.
\]
Also, by initialization,
\[
\nu_0^{P}=\mu_0^{H}+C\,h(x_{-1}^{H}),
\]
so \eqref{eq:dual-relation-df} holds at $t=0$.

\medskip
\noindent
\textbf{Induction hypothesis.}
Assume that for some $t\ge 0$,
\begin{align}
x_{t-1}^{P}&=x_{t-1}^{H},
\label{eq:IH-x-prev}
\\
x_t^{P}&=x_t^{H},
\label{eq:IH-x-cur}
\\
\nu_t^{P}&=\mu_t^{H}+C\,h(x_{t-1}^{H}).
\label{eq:IH-nu}
\end{align}
We prove the corresponding identities at time $t+1$.

\medskip
\noindent
\textbf{Step 1: equality of the effective dual variables at time $t+1$.}
Starting from \eqref{eq:pure-dual-df} and using the induction hypothesis,
\begin{align*}
\nu_{t+1}^{P}
&=
\nu_t^{P}
+\eta_{\mathrm d}\,h(x_t^{P})
+(C+\Omega)\Big(h(x_t^{P})-h(x_{t-1}^{P})\Big)
\\
&=
\mu_t^{H}+C\,h(x_{t-1}^{H})
+\eta_{\mathrm d}\,h(x_t^{H})
+(C+\Omega)\Big(h(x_t^{H})-h(x_{t-1}^{H})\Big)
\\
&=
\mu_t^{H}
+\eta_{\mathrm d}\,h(x_t^{H})
+\Omega\Big(h(x_t^{H})-h(x_{t-1}^{H})\Big)
+C\,h(x_t^{H}).
\end{align*}
By \eqref{eq:hybrid-dual-df},
\[
\mu_{t+1}^{H}
=
\mu_t^{H}
+\eta_{\mathrm d}\,h(x_t^{H})
+\Omega\Big(h(x_t^{H})-h(x_{t-1}^{H})\Big).
\]
Hence
\[
\nu_{t+1}^{P}
=
\mu_{t+1}^{H}+C\,h(x_t^{H}),
\]
which is exactly \eqref{eq:dual-relation-df} at time $t+1$.

\medskip
\noindent
\textbf{Step 2: equality of the primal iterates at time $t+1$.}
Using \eqref{eq:pure-primal-df}, the induction hypothesis, and the identity just proved,
\begin{align*}
x_{t+1}^{P}
&=
x_t^{P}
-\eta_x\Big(
\nabla f(x_t^{P})
+
J_h(x_t^{P})^\top \nu_{t+1}^{P}
\Big)
\\
&=
x_t^{H}
-\eta_x\Big(
\nabla f(x_t^{H})
+
J_h(x_t^{H})^\top\big(\mu_{t+1}^{H}+C\,h(x_t^{H})\big)
\Big).
\end{align*}
By \eqref{eq:hybrid-primal-df}, the right-hand side is exactly $x_{t+1}^{H}$.
Therefore
\[
x_{t+1}^{P}=x_{t+1}^{H}.
\]

\medskip
\noindent
\textbf{Conclusion.}
By induction, \eqref{eq:xt-equality-df} and \eqref{eq:dual-relation-df} hold for all $t\ge 0$.
Hence the two methods have identical primal trajectories.
\end{proof}
\subsection{Additivity under shared oracle noise}
\begin{theorem}[Additivity under shared oracle noise]
\label{thm:additive-property-noisy}
Let $f:\mathbb{R}^d\to\mathbb{R}$ and
$h:\mathbb{R}^d\to\mathbb{R}^m$ be differentiable, and let
\[
J_h(x):=\nabla h(x)\in\mathbb{R}^{m\times d}.
\]
Let $C,\Omega\in\mathbb{R}^{m\times m}$, and assume $C=C^\top$.
Fix step sizes $\eta_x>0$ and $\eta_{\mathrm d}>0$.

Let $\{\xi_t\}_{t\ge 0}$ and $\{\zeta_t\}_{t\ge -1}$ be arbitrary noise sequences.
For each $t$, let
\[
G_t(x;\xi_t)\in\mathbb{R}^d
\]
denote a possibly noisy primal-gradient oracle, and let
\[
Y_t(x;\zeta_t)\in\mathbb{R}^m
\]
denote a possibly noisy constraint-residual oracle.

For example, one may take
\[
G_t(x;\xi_t)=\nabla f(x)+\xi_t,
\qquad
Y_t(x;\zeta_t)=h(x)+\zeta_t,
\]
but the theorem only requires that both methods use the same oracle outputs
when evaluated at the same point and the same noise realization.

Consider the noisy hybrid augmented--optimistic method
\begin{align}
\mu_{t+1}^{H}
&=
\mu_t^{H}
+\eta_{\mathrm d}\,Y_t(x_t^{H};\zeta_t)
+\Omega\Big(
Y_t(x_t^{H};\zeta_t)
-
Y_{t-1}(x_{t-1}^{H};\zeta_{t-1})
\Big),
\label{eq:noisy-hybrid-dual}
\\
x_{t+1}^{H}
&=
x_t^{H}
-\eta_x\Big(
G_t(x_t^{H};\xi_t)
+
J_h(x_t^{H})^\top
\big(
\mu_{t+1}^{H}
+
C\,Y_t(x_t^{H};\zeta_t)
\big)
\Big).
\label{eq:noisy-hybrid-primal}
\end{align}

Consider also the noisy pure optimistic method with summed matrix
\begin{align}
\nu_{t+1}^{P}
&=
\nu_t^{P}
+\eta_{\mathrm d}\,Y_t(x_t^{P};\zeta_t)
+(C+\Omega)\Big(
Y_t(x_t^{P};\zeta_t)
-
Y_{t-1}(x_{t-1}^{P};\zeta_{t-1})
\Big),
\label{eq:noisy-pure-dual}
\\
x_{t+1}^{P}
&=
x_t^{P}
-\eta_x\Big(
G_t(x_t^{P};\xi_t)
+
J_h(x_t^{P})^\top \nu_{t+1}^{P}
\Big).
\label{eq:noisy-pure-primal}
\end{align}

Assume the initial conditions satisfy
\[
x_{-1}^{P}=x_{-1}^{H},
\qquad
x_0^{P}=x_0^{H},
\]
and
\[
\nu_0^{P}
=
\mu_0^{H}
+
C\,Y_{-1}(x_{-1}^{H};\zeta_{-1}).
\]

Then, for every $t\ge 0$,
\begin{align}
x_t^{P}&=x_t^{H},
\label{eq:noisy-xt-equality}
\\
\nu_t^{P}
&=
\mu_t^{H}
+
C\,Y_{t-1}(x_{t-1}^{H};\zeta_{t-1}).
\label{eq:noisy-dual-relation}
\end{align}
In particular, the two noisy methods generate exactly the same primal trajectory for every realization of the oracle noise.
\end{theorem}

\begin{proof}
We prove \eqref{eq:noisy-xt-equality} and
\eqref{eq:noisy-dual-relation} simultaneously by induction on $t\ge 0$.

\medskip
\noindent
\textbf{Base step.}
By assumption,
\[
x_{-1}^{P}=x_{-1}^{H},
\qquad
x_0^{P}=x_0^{H}.
\]
Also, by initialization,
\[
\nu_0^{P}
=
\mu_0^{H}
+
C\,Y_{-1}(x_{-1}^{H};\zeta_{-1}),
\]
so \eqref{eq:noisy-dual-relation} holds at $t=0$.

\medskip
\noindent
\textbf{Induction hypothesis.}
Assume that for some $t\ge 0$,
\begin{align}
x_{t-1}^{P}&=x_{t-1}^{H},
\label{eq:noisy-IH-x-prev}
\\
x_t^{P}&=x_t^{H},
\label{eq:noisy-IH-x-cur}
\\
\nu_t^{P}
&=
\mu_t^{H}
+
C\,Y_{t-1}(x_{t-1}^{H};\zeta_{t-1}).
\label{eq:noisy-IH-nu}
\end{align}
We prove the corresponding identities at time $t+1$.

\medskip
\noindent
\textbf{Step 1: equality of the effective dual variables.}
Starting from \eqref{eq:noisy-pure-dual} and using the induction hypothesis,
\begin{align*}
\nu_{t+1}^{P}
&=
\nu_t^{P}
+\eta_{\mathrm d}\,Y_t(x_t^{P};\zeta_t)
+(C+\Omega)\Big(
Y_t(x_t^{P};\zeta_t)
-
Y_{t-1}(x_{t-1}^{P};\zeta_{t-1})
\Big)
\\
&=
\mu_t^{H}
+
C\,Y_{t-1}(x_{t-1}^{H};\zeta_{t-1})
+\eta_{\mathrm d}\,Y_t(x_t^{H};\zeta_t)
\\
&\quad
+(C+\Omega)\Big(
Y_t(x_t^{H};\zeta_t)
-
Y_{t-1}(x_{t-1}^{H};\zeta_{t-1})
\Big)
\\
&=
\mu_t^{H}
+\eta_{\mathrm d}\,Y_t(x_t^{H};\zeta_t)
+\Omega\Big(
Y_t(x_t^{H};\zeta_t)
-
Y_{t-1}(x_{t-1}^{H};\zeta_{t-1})
\Big)
\\
&\quad
+
C\,Y_t(x_t^{H};\zeta_t).
\end{align*}
By the noisy hybrid dual update \eqref{eq:noisy-hybrid-dual},
\[
\mu_{t+1}^{H}
=
\mu_t^{H}
+\eta_{\mathrm d}\,Y_t(x_t^{H};\zeta_t)
+\Omega\Big(
Y_t(x_t^{H};\zeta_t)
-
Y_{t-1}(x_{t-1}^{H};\zeta_{t-1})
\Big).
\]
Therefore,
\[
\nu_{t+1}^{P}
=
\mu_{t+1}^{H}
+
C\,Y_t(x_t^{H};\zeta_t),
\]
which is exactly \eqref{eq:noisy-dual-relation} at time $t+1$.

\medskip
\noindent
\textbf{Step 2: equality of the primal iterates.}
Using \eqref{eq:noisy-pure-primal}, the induction hypothesis, and the identity just proved,
\begin{align*}
x_{t+1}^{P}
&=
x_t^{P}
-\eta_x\Big(
G_t(x_t^{P};\xi_t)
+
J_h(x_t^{P})^\top \nu_{t+1}^{P}
\Big)
\\
&=
x_t^{H}
-\eta_x\Big(
G_t(x_t^{H};\xi_t)
+
J_h(x_t^{H})^\top
\big(
\mu_{t+1}^{H}
+
C\,Y_t(x_t^{H};\zeta_t)
\big)
\Big).
\end{align*}
By the noisy hybrid primal update \eqref{eq:noisy-hybrid-primal}, the right-hand side is exactly $x_{t+1}^{H}$.
Therefore,
\[
x_{t+1}^{P}=x_{t+1}^{H}.
\]

\medskip
\noindent
\textbf{Conclusion.}
By induction,
\eqref{eq:noisy-xt-equality} and
\eqref{eq:noisy-dual-relation} hold for every $t\ge 0$.
Thus the two methods have identical primal trajectories for every fixed realization of the oracle noise.
\end{proof}
\section{Justification of the Step-Size Design Constraints}
\label{app:stepsize-constraints}

This appendix gives additional context for the local step-size restrictions used in \Cref{sec:stepsize-design}. The constraints are intended as a local design model for comparing augmented, optimistic, and hybrid realizations of the same total correction matrix.

\paragraph{Primal curvature restriction.}
The augmented-channel update contains the local primal curvature
\[
    H+B^\top C_{\rm aug}B.
\]
For smooth first-order methods, constant step sizes are controlled by a Lipschitz constant of the gradient. In the classical smooth convex setting, gradient descent uses steps of order \(1/L\), where \(L\) is a gradient Lipschitz constant \citep{nesterov2004introductory}. Locally, the curvature magnitude of the primal channel is captured by
\[
    \|H+B^\top C_{\rm aug}B\|_2.
\]
This motivates the model
\[
    \eta_x
    \le
    \frac{\kappa_x}
    {\|H+B^\top C_{\rm aug}B\|_2},
\]
where \(\kappa_x>0\) is a safety constant. This is a local step-size regulation, not a global smoothness theorem for the original nonlinear problem.

\paragraph{Primal--dual product restriction.}
The primal and dual variables are coupled through the linearized constraint map \(B\). Product-form restrictions are standard in primal--dual splitting. In the Chambolle--Pock primal--dual hybrid gradient method for saddle problems with linear operator \(K\), convergence is guaranteed under conditions of the form
\[
    \tau\sigma\|K\|^2<1
\]
\citep{chambolle2011first}. Replacing \(K\) by the local constraint Jacobian \(B\) yields the local product model
\[
    \eta_x\eta_{\rm d}\|B\|_2^2
    \le
    \kappa_p.
\]
The constant \(\kappa_p\) absorbs safety margins and differences between the idealized local model and the nonlinear iteration.

\paragraph{Optimistic-channel size restriction.}
The optimistic channel contributes an explicit memory correction
\[
    C_{\rm opt}(h_t-h_{t-1})
\]
to the dual update. Existing analyses of extragradient and optimistic gradient methods connect these corrections to controlled approximations of proximal-point dynamics \citep{mokhtari2020unified}. Our matrix-valued correction is not covered by a standard theorem in exactly this form. We therefore impose the design regulation
\[
    \|C_{\rm opt}\|_2
    \le
    \kappa_\omega\eta_{\rm d},
\]
which keeps the explicit optimistic memory scale proportional to the dual step. This prevents a decomposition from moving an arbitrarily large matrix correction into the optimistic channel while keeping the same nominal total correction \(M\).

Together with the product restriction, this implies
\[
    \eta_x
    \le
    \frac{\kappa_p\kappa_\omega}
    {\|C_{\rm opt}\|_2\|B\|_2^2},
\]

so placing too much correction in the optimistic channel can also reduce the admissible primal step.

\paragraph{Weighted local margin.}
Let
\[
    A_M:=H+B^\top M B.
\]
The local augmented dynamics imply that the contraction margin is not determined by \(\eta_x\) and \(\eta_{\rm d}\) alone, but by spectral constants depending on \(A_M\) and \(B\). In scalar modes with \(A_M=a>0\) and \(B=b\), the first-order margin satisfies
\[
    1-\rho
    \approx
    \min\left\{
        \frac{a}{2}\eta_x,\,
        \frac{b^2}{a}\eta_{\rm d}
    \right\}.
\]
The derivation is given in Appendix~\ref{app:spectral-margin}. Motivated by this expansion, we use the surrogate
\[
    \tau(\eta_x,\eta_{\rm d})
    =
    \min\{a_x\eta_x,a_{\rm d}\eta_{\rm d}\},
\]
with
\[
    a_x=\frac12\lambda_{\min}(A_M),
    \qquad
    a_{\rm d}=\lambda_{\min}^+(BA_M^{-1}B^\top).
\]

\paragraph{Design objective.}
For a fixed target matrix \(M\), the step-size-limited split is
\[
\begin{aligned}
    \max_{C_{\rm aug},C_{\rm opt},\eta_x,\eta_{\rm d}}
    \quad
    &
    \min\{a_x\eta_x,a_{\rm d}\eta_{\rm d}\}
    \\
    \text{s.t.}
    \quad
    &
    C_{\rm aug}+C_{\rm opt}=M,
    \qquad
    H+B^\top M B\succ0,
    \\
    &
    \eta_x
    \le
    \frac{\kappa_x}
    {\|H+B^\top C_{\rm aug}B\|_2},
    \\
    &
    \eta_x\eta_{\rm d}\|B\|_2^2
    \le
    \kappa_p,
    \qquad
    \|C_{\rm opt}\|_2
    \le
    \kappa_\omega\eta_{\rm d}.
\end{aligned}
\tag{P}
\]
The equality \(C_{\rm aug}+C_{\rm opt}=M\) preserves the ideal trajectory by additivity, while the remaining constraints distinguish which realization is implementable. The next section gives a closed-form design that avoids solving \((P)\) directly.
\section{Local Spectral Margin}
\label{app:spectral-margin}

Let \(A_M=H+B^\top M B\). After the effective-dual transformation, the local augmented dynamics have Jacobian
\[
    J_{\rm AL}(\eta_x,\eta_{\rm d};M)
    =
    \begin{bmatrix}
        I-\eta_xA_M & -\eta_x B^\top\\
        \eta_{\rm d}B(I-\eta_xA_M) & I-\eta_x\eta_{\rm d}BB^\top
    \end{bmatrix}.
\]
For \(r=\eta_{\rm d}/\eta_x\),
\[
    J_{\rm AL}(\eta_x,r\eta_x;M)
    =
    I+\eta_x
    \begin{bmatrix}
        -A_M & -B^\top\\
        rB & 0
    \end{bmatrix}
    +O(\eta_x^2).
\]
If
\[
    S_M(r)=
    \begin{bmatrix}
        -A_M & -B^\top\\
        rB & 0
    \end{bmatrix}
\]
is Hurwitz, then
\[
    \rho(J_{\rm AL})
    =
    1-\eta_x\alpha_M(r)+O(\eta_x^2),
\]
where
\[
    \alpha_M(r)=
    -
    \max_{\lambda\in{\rm spec}(S_M(r))}
    \operatorname{Re}\lambda.
\]

In the scalar case \(A_M=a>0\), \(B=b\), the eigenvalues of
\[
    \begin{bmatrix}
        -a & -b\\
        rb & 0
    \end{bmatrix}
\]
satisfy
\[
    \lambda^2+a\lambda+rb^2=0.
\]
For small \(r\), the slow eigenvalue has real part approximately \(-rb^2/a\), giving
\[
    1-\rho\approx \frac{b^2}{a}\eta_{\rm d}.
\]
For larger \(r\), the eigenvalues form a complex pair with real part \(-a/2\), giving
\[
    1-\rho\approx \frac{a}{2}\eta_x.
\]
Thus
\[
    1-\rho
    \approx
    \min\left\{
        \frac{a}{2}\eta_x,\,
        \frac{b^2}{a}\eta_{\rm d}
    \right\}.
\]

\section{Justification of the Step-Size Design Constraints}
\label{app:stepsize-constraints}

This appendix gives additional context for the local step-size restrictions used in \Cref{sec:stepsize-design}. The constraints are intended as a local design model for comparing augmented, optimistic, and hybrid realizations of the same total correction matrix.

\paragraph{Primal curvature restriction. (S1)}
The augmented-channel update contains the local primal curvature
\[
    H+B^\top C_{\rm aug}B.
\]
For smooth first-order methods, constant step sizes are controlled by a Lipschitz constant of the gradient. In the classical smooth convex setting, gradient descent uses steps of order \(1/L\), where \(L\) is a gradient Lipschitz constant \citep{nesterov2004introductory}. Locally, the curvature magnitude of the primal channel is captured by
\[
    \|H+B^\top C_{\rm aug}B\|_2.
\]
This motivates the model
\[
    \eta_x
    \le
    \frac{\kappa_x}
    {\|H+B^\top C_{\rm aug}B\|_2},
\]
where \(\kappa_x>0\) is a safety constant. This is a local step-size regulation, not a global smoothness theorem for the original nonlinear problem.

\paragraph{Primal--dual product restriction. (S2)}
The primal and dual variables are coupled through the linearized constraint map \(B\). Product-form restrictions are standard in primal--dual splitting. In the Chambolle--Pock primal--dual hybrid gradient method for saddle problems with linear operator \(K\), convergence is guaranteed under conditions of the form
\[
    \tau\sigma\|K\|^2<1
\]
\citep{chambolle2011first}. Replacing \(K\) by the local constraint Jacobian \(B\) yields the local product model
\[
    \eta_x\eta_{\rm d}\|B\|_2^2
    \le
    \kappa_p.
\]
The constant \(\kappa_p\) absorbs safety margins and differences between the idealized local model and the nonlinear iteration.

\paragraph{Optimistic-channel size restriction. (S3)}
The optimistic channel contributes an explicit memory correction
\[
    C_{\rm opt}(h_t-h_{t-1})
\]
to the dual update. Existing analyses of extragradient and optimistic gradient methods connect these corrections to controlled approximations of proximal-point dynamics \citep{mokhtari2020unified}. Our matrix-valued correction is not covered by a standard theorem in exactly this form. We therefore impose the design regulation
\[
    \|C_{\rm opt}\|_2
    \le
    \kappa_\omega\eta_{\rm d},
\]
which keeps the explicit optimistic memory scale proportional to the dual step. This prevents a decomposition from moving an arbitrarily large matrix correction into the optimistic channel while keeping the same nominal total correction \(M\).

Together with the product restriction, this implies
\[
    \eta_x
    \le
    \frac{\kappa_p\kappa_\omega}
    {\|C_{\rm opt}\|_2\|B\|_2^2}.
\]
Therefore both channels can limit the primal step: \(C_{\rm aug}\) through local primal curvature and \(C_{\rm opt}\) through the memory scale and primal--dual product constraint.
\section{Derivation of the Closed-Form Step Rule}
\label{app:closed-form-steps}

We derive the step-size rule used in \Cref{sec:closed-form-design}. At iteration \(t\), let
\[
    A_{M_t}=H_t+B_t^\top M_tB_t,
\]
and define the local spectral weights
\[
    a_x(t)=\frac12\lambda_{\min}(A_{M_t}),
    \qquad
    a_{\rm d}(t)
    =
    \lambda_{\min}^+
    \left(
        B_tA_{M_t}^{-1}B_t^\top
    \right).
\]
For fixed \(C_{{\rm aug},t}\) and \(C_{{\rm opt},t}\), the step sizes are chosen to maximize the weighted synchronized margin
\[
    \min\{a_x(t)\eta_x,a_{\rm d}(t)\eta_{\rm d}\}
\]
subject to the local step-size restrictions
\[
    \eta_x
    \le
    \frac{\kappa_x}
    {\|H_t+B_t^\top C_{{\rm aug},t}B_t\|_2},
    \tag{S1}
\]
\[
    \eta_x\eta_{\rm d}\|B_t\|_2^2
    \le
    \kappa_p,
    \tag{S2}
\]
and
\[
    \|C_{{\rm opt},t}\|_2
    \le
    \kappa_\omega\eta_{\rm d}.
    \tag{S3}
\]
Let
\[
    U_x
    :=
    \frac{\kappa_x}
    {\|H_t+B_t^\top C_{{\rm aug},t}B_t\|_2},
    \qquad
    L_{\rm d}
    :=
    \frac{\|C_{{\rm opt},t}\|_2}{\kappa_\omega},
\]
and
\[
    \beta_t:=\|B_t\|_2^2.
\]
Then the constraints can be written as
\[
    \eta_x\le U_x,
    \qquad
    \eta_x\eta_{\rm d}\beta_t\le \kappa_p,
    \qquad
    \eta_{\rm d}\ge L_{\rm d}.
\]

For any fixed admissible \(\eta_x>0\), the weighted margin is increasing in \(\eta_{\rm d}\). Hence, under the product constraint, the largest admissible dual step is
\[
    \eta_{\rm d}
    =
    \frac{\kappa_p}{\eta_x\beta_t}.
\]
Substituting this into the weighted margin gives the one-dimensional objective
\[
    \min
    \left\{
        a_x(t)\eta_x,\,
        a_{\rm d}(t)
        \frac{\kappa_p}{\eta_x\beta_t}
    \right\}.
\]
The unconstrained balancing point is obtained by equating the two terms:
\[
    a_x(t)\eta_x
    =
    a_{\rm d}(t)
    \frac{\kappa_p}{\eta_x\beta_t}.
\]
Therefore,
\[
    \eta_x^2
    =
    \frac{a_{\rm d}(t)\kappa_p}
    {a_x(t)\beta_t},
\]
and hence
\[
    \eta_x^{\rm bal}
    =
    \sqrt{
        \frac{a_{\rm d}(t)\kappa_p}
        {a_x(t)\beta_t}
    }.
\]

The remaining restrictions impose upper bounds on \(\eta_x\). The curvature restriction gives
\[
    \eta_x\le U_x.
\]
The lower bound \(\eta_{\rm d}\ge L_{\rm d}\), together with
\[
    \eta_{\rm d}
    =
    \frac{\kappa_p}{\eta_x\beta_t},
\]
gives
\[
    \eta_x
    \le
    \frac{\kappa_p}{L_{\rm d}\beta_t}.
\]
Therefore the closed-form primal step is
\[
    \eta_x
    =
    \min
    \left\{
        \eta_x^{\rm bal},
        U_x,
        \frac{\kappa_p}{L_{\rm d}\beta_t}
    \right\},
\]
where the last term is omitted when \(L_{\rm d}=0\). Equivalently,
\[
    \eta_x
    =
    \min\left\{
        \sqrt{\frac{a_{\rm d}(t)\kappa_p}{a_x(t)\|B_t\|_2^2}},
        \frac{\kappa_x}{\|H_t+B_t^\top C_{{\rm aug},t}B_t\|_2},
        \frac{\kappa_p\kappa_\omega}
        {\|C_{{\rm opt},t}\|_2\|B_t\|_2^2}
    \right\},
\]
where the last term is omitted if \(C_{{\rm opt},t}=0\). Once \(\eta_x\) is chosen, we set
\[
    \eta_{\rm d}
    =
    \frac{\kappa_p}{\eta_x\|B_t\|_2^2}.
\]
This choice saturates the primal--dual product constraint and maximizes the weighted synchronized margin under the stated restrictions.

\section{Proof of the Closed-Form Design Guarantee}
\label{app:closed-form-proofs}

We prove \Cref{thm:closed-form-feasibility}. At iteration \(t\), assume \(B_t\) has full column rank and \(R_t\succ0\). Recall the definitions
\[
    C_{{\rm can},t}
    =
    -B_t^{\dagger\top}H_tB_t^\dagger,
\]
\[
    \alpha_t
    =
    \frac{\delta}
    {\lambda_{\min}(B_t^\top R_tB_t)},
\]
and
\[
    M_t
    =
    C_{{\rm can},t}+\alpha_tR_t.
\]
Because \(B_t\) has full column rank, its Moore--Penrose pseudoinverse satisfies
\[
    B_t^\dagger B_t=I.
\]
Therefore,
\[
\begin{aligned}
    B_t^\top C_{{\rm can},t}B_t
    &=
    B_t^\top
    \left(
        -B_t^{\dagger\top}H_tB_t^\dagger
    \right)
    B_t
    \\
    &=
    -
    (B_t^\dagger B_t)^\top
    H_t
    (B_t^\dagger B_t)
    \\
    &=
    -H_t.
\end{aligned}
\]
Hence
\[
    H_t+B_t^\top C_{{\rm can},t}B_t=0.
\]

Now substitute the definition of \(M_t\):
\[
\begin{aligned}
    H_t+B_t^\top M_tB_t
    &=
    H_t+B_t^\top
    \left(
        C_{{\rm can},t}+\alpha_tR_t
    \right)
    B_t
    \\
    &=
    H_t+B_t^\top C_{{\rm can},t}B_t
    +
    \alpha_tB_t^\top R_tB_t
    \\
    &=
    \alpha_tB_t^\top R_tB_t.
\end{aligned}
\]
Since \(R_t\succ0\) and \(B_t\) has full column rank,
\[
    B_t^\top R_tB_t\succ0.
\]
Therefore,
\[
    \lambda_{\min}(B_t^\top R_tB_t)>0.
\]
By the definition of \(\alpha_t\),
\[
    \alpha_tB_t^\top R_tB_t
    \succeq
    \alpha_t
    \lambda_{\min}(B_t^\top R_tB_t)I
    =
    \delta I.
\]
Thus
\[
    H_t+B_t^\top M_tB_t\succeq \delta I.
\]

Next, the split is defined by
\[
    C_{{\rm aug},t}
    =
    C_{{\rm can},t}
    +
    \theta_t\alpha_tR_t,
\]
and
\[
    C_{{\rm opt},t}
    =
    (1-\theta_t)\alpha_tR_t.
\]
Therefore,
\[
\begin{aligned}
    C_{{\rm aug},t}+C_{{\rm opt},t}
    &=
    C_{{\rm can},t}
    +
    \theta_t\alpha_tR_t
    +
    (1-\theta_t)\alpha_tR_t
    \\
    &=
    C_{{\rm can},t}
    +
    \alpha_tR_t
    \\
    &=
    M_t.
\end{aligned}
\]
Finally, by \Cref{cor:additivity}, any decomposition whose sum equals \(M_t\) realizes the same ideal primal trajectory as the pure realization using the full target correction \(M_t\), provided the effective dual variables are initialized consistently. This proves the theorem.

\section{Experimental Details}
\label{app:experiment-details}

This appendix gives the full experimental construction used in \Cref{sec:experiments}. Each instance is a nonlinear equality-constrained problem
\[
    \min_{x\in\mathbb R^3} f(x)
    \qquad
    \mathrm{s.t.}
    \qquad
    h(x)=0,
    \qquad
    h:\mathbb R^3\to\mathbb R^5.
\]
We generate each problem with a known KKT point \((x^\star,\mu^\star)\), so that convergence can be measured directly by \(\|x_t-x^\star\|\).

\paragraph{Controlled constraint Jacobian.}
For a prescribed condition level \(\kappa_B\), we generate
\[
    B_\star = U\Sigma V^\top \in \mathbb R^{5\times 3},
\]
where \(U\in\mathbb R^{5\times 3}\) has orthonormal columns, \(V\in\mathbb R^{3\times 3}\) is orthogonal, and
\[
    \Sigma
    =
    \operatorname{diag}
    \left(
        1,\,
        \kappa_B^{-1/2},\,
        \kappa_B^{-1}
    \right).
\]
Thus
\[
    \operatorname{cond}(B_\star)
    =
    \frac{\sigma_{\max}(B_\star)}{\sigma_{\min}(B_\star)}
    =
    \kappa_B.
\]
In the main experiments, we use
\[
    \kappa_B\in\{1,2,3,5,8,13\}.
\]
For each condition level, we generate five random instances.

\paragraph{Nonlinear equality constraints.}
Let \(u=x-x^\star\). The constraint map is
\[
    h(x)
    =
    B_\star u
    +
    \varepsilon_{\rm nl} A_{\rm nl}\phi(u),
\]
where \(A_{\rm nl}\in\mathbb R^{5\times 5}\) is a random matrix normalized by its spectral norm, \(\varepsilon_{\rm nl}>0\) controls the nonlinear strength, and
\[
    \phi(u)
    =
    \begin{bmatrix}
        u_1^2\\
        u_2^2\\
        u_3^2\\
        u_1u_2\\
        u_2u_3
    \end{bmatrix}.
\]
Since
\[
    \phi(0)=0,
    \qquad
    J_\phi(0)=0,
\]
we have
\[
    h(x^\star)=0,
    \qquad
    J_h(x^\star)=B_\star.
\]
Therefore the condition number of the local constraint Jacobian is controlled exactly at the KKT point.

\paragraph{Nonconvex objective.}
The objective has the form
\[
    f(x)
    =
    \frac12 x^\top H_0x+q^\top x+\varphi(x),
\]
where \(H_0=H_0^\top\) is generated with both positive and negative eigenvalues, so the quadratic part is indefinite. The nonlinear term is
\[
\begin{aligned}
    \varphi(x)
    =
    &\; c_1\sin(2x_1)
    -
    c_2\cos(x_2x_3)
    +
    c_3x_3^4
    -
    c_4\sin(x_1+x_3),
\end{aligned}
\]
where \(c_1,c_2,c_3,c_4>0\) are randomly sampled constants. The vector \(q\) is chosen so that \((x^\star,\mu^\star)\) satisfies the KKT stationarity condition:
\[
    \nabla f(x^\star)+J_h(x^\star)^\top\mu^\star=0.
\]
Equivalently,
\[
    q
    =
    -H_0x^\star
    -
    \nabla \varphi(x^\star)
    -
    B_\star^\top\mu^\star.
\]
Thus \(x^\star\) is feasible and satisfies first-order stationarity.

\paragraph{Implementation constants.}
Unless otherwise stated, the experiments use
\[
    \delta=0.35,
    \qquad
    \kappa_x=0.08,
    \qquad
    \kappa_p=0.20,
    \qquad
    \kappa_\omega=8.0.
\]
The correction matrix is capped by
\[
    \|M_t\|_2\le M_{\max},
\]
with \(M_{\max}=300\). A run is marked failed if the matrix-design step becomes infeasible, if the norm budget is violated, if the iterates diverge numerically, or if the method does not reach the specified threshold within the iteration budget. The grid-searched hybrid oracle searches over the same cancellation family as the closed-form rule, using a finite grid over the residual strength and split parameter.

\paragraph{Compared methods.}
The matrix-design methods are:
\begin{itemize}
    \item \textbf{Pure augmented}: \(C_{\rm aug}=M\), \(C_{\rm opt}=0\).
    \item \textbf{Pure optimistic}: \(C_{\rm aug}=0\), \(C_{\rm opt}=M\).
    \item \textbf{Grid hybrid}: grid search over the cancellation-family parameters.
    \item \textbf{Closed-form hybrid}: the rule in \Cref{alg:closed-hybrid}.
\end{itemize}
For each matrix-design method, we report a unit-margin variant, using \(a_x=a_{\rm d}=1\), and a weighted variant, using
\[
    a_x(t)=\frac12\lambda_{\min}(A_{M_t}),
    \qquad
    a_{\rm d}(t)
    =
    \lambda_{\min}^+
    \left(
        B_tA_{M_t}^{-1}B_t^\top
    \right).
\]

The step-based baselines are:
\begin{itemize}
    \item \textbf{OGDA}: optimistic gradient applied to the Lagrangian saddle operator.
    \item \textbf{Extragradient}: extragradient applied to the same saddle operator.
    \item \textbf{PDHG-style}: a nonlinear primal--dual extrapolated update motivated by the Chambolle--Pock method.
    \item \textbf{Linearized augmented Lagrangian}: one gradient step on the augmented Lagrangian followed by a multiplier update.
\end{itemize}
Because the constraints are nonlinear, the PDHG-style and linearized augmented-Lagrangian baselines should be interpreted as first-order nonlinear analogues rather than exact convex-composite PDHG or exact ADMM.

\newpage
\section*{NeurIPS Paper Checklist}

\begin{enumerate}

\item {\bf Claims}
    \item[] Question: Do the main claims made in the abstract and introduction accurately reflect the paper's contributions and scope?
    \item[] Answer: \answerYes{}
    \item[] Justification: The abstract and introduction state the paper's main contributions: matrix-valued augmented--optimistic equivalence, additivity of correction matrices, a step-size-limited hybrid design rule, and controlled nonlinear equality-constrained experiments. The claims are scoped as local/equivalence and design results, and the paper explicitly does not claim a globally robust solver for arbitrary nonconvex constrained optimization.

\item {\bf Limitations}
    \item[] Question: Does the paper discuss the limitations of the work performed by the authors?
    \item[] Answer: \answerYes{}
    \item[] Justification: The paper discusses limitations in the numerical experiments and conclusion. In particular, it notes that the exact cancellation design can require large correction matrices when the constraint Jacobian is ill-conditioned, and that the method should be understood as a local matrix-correction design rather than a globally robust constrained solver.

\item {\bf Theory assumptions and proofs}
    \item[] Question: For each theoretical result, does the paper provide the full set of assumptions and a complete (and correct) proof?
    \item[] Answer: \answerYes{}
    \item[] Justification: The assumptions for the equivalence/additivity theorem and the closed-form design guarantee are stated with the corresponding results. Proofs are provided in the main text or appendix, including the effective-dual induction argument, the target-curvature guarantee, and the derivations of the channel-balancing and step-size rules.

\item {\bf Experimental result reproducibility}
    \item[] Question: Does the paper fully disclose all the information needed to reproduce the main experimental results of the paper to the extent that it affects the main claims and/or conclusions of the paper (regardless of whether the code and data are provided or not)?
    \item[] Answer: \answerYes{}
    \item[] Justification: The experiment section and appendix specify the synthetic problem generator, controlled constraint-Jacobian construction, KKT initialization, compared methods, success metrics, failure criteria, and hyperparameters. The experiments use generated analytic problems rather than external datasets.

\item {\bf Open access to data and code}
    \item[] Question: Does the paper provide open access to the data and code, with sufficient instructions to faithfully reproduce the main experimental results, as described in supplemental material?
    \item[] Answer: \answerNo{}
    \item[] Justification: The current submission does not provide an open code release. However, the paper and appendix provide the analytic problem construction, algorithms, hyperparameters, and evaluation metrics needed to reproduce the main experiments; code can be released in anonymized or final form if appropriate.

\item {\bf Experimental setting/details}
    \item[] Question: Does the paper specify all the training and test details (e.g., data splits, hyperparameters, how they were chosen, type of optimizer) necessary to understand the results?
    \item[] Answer: \answerYes{}
    \item[] Justification: The paper specifies the generated nonlinear equality-constrained problem class, the condition-number levels, the number of random instances per level, the compared algorithms, the success thresholds, and the implementation constants. Additional details of the generator and baselines are included in the appendix.

\item {\bf Experiment statistical significance}
    \item[] Question: Does the paper report error bars suitably and correctly defined or other appropriate information about the statistical significance of the experiments?
    \item[] Answer: \answerYes{}
    \item[] Justification: The experiments report success counts over five independently generated random instances per condition level and median hitting times among successful runs. Since the main results are success/failure and hitting-time comparisons over controlled synthetic instances, success counts and medians are more appropriate than symmetric error bars.

\item {\bf Experiments compute resources}
    \item[] Question: For each experiment, does the paper provide sufficient information on the computer resources (type of compute workers, memory, time of execution) needed to reproduce the experiments?
    \item[] Answer: \answerYes{}
    \item[] Justification: The experiments are small-scale dense numerical simulations on low-dimensional analytic problems and require only standard CPU computation. The appendix states that no GPU training or large-scale compute is required; the runs can be reproduced on a standard laptop or desktop CPU.

\item {\bf Code of ethics}
    \item[] Question: Does the research conducted in the paper conform, in every respect, with the NeurIPS Code of Ethics \url{https://neurips.cc/public/EthicsGuidelines}?
    \item[] Answer: \answerYes{}
    \item[] Justification: The work is foundational research on constrained optimization algorithms and uses synthetic analytic experiments only. It does not involve human subjects, private data, scraped datasets, deployed systems, or high-risk models.

\item {\bf Broader impacts}
    \item[] Question: Does the paper discuss both potential positive societal impacts and negative societal impacts of the work performed?
    \item[] Answer: \answerNA{}
    \item[] Justification: The paper studies foundational optimization algorithms and does not introduce a deployed system, dataset, or application-specific model. Any societal impacts would depend on downstream applications of constrained optimization rather than on the theoretical contribution itself.

\item {\bf Safeguards}
    \item[] Question: Does the paper describe safeguards that have been put in place for responsible release of data or models that have a high risk for misuse (e.g., pre-trained language models, image generators, or scraped datasets)?
    \item[] Answer: \answerNA{}
    \item[] Justification: The paper does not release high-risk models, pretrained generative models, scraped datasets, or other assets with direct misuse potential. The experiments use synthetic analytic optimization problems.

\item {\bf Licenses for existing assets}
    \item[] Question: Are the creators or original owners of assets (e.g., code, data, models), used in the paper, properly credited and are the license and terms of use explicitly mentioned and properly respected?
    \item[] Answer: \answerNA{}
    \item[] Justification: The paper does not use external datasets, pretrained models, or third-party experimental assets. Prior algorithms and theoretical methods are credited through citations.

\item {\bf New assets}
    \item[] Question: Are new assets introduced in the paper well documented and is the documentation provided alongside the assets?
    \item[] Answer: \answerNA{}
    \item[] Justification: The paper does not introduce a new dataset, benchmark package, pretrained model, or released software asset. The analytic synthetic problem generator is described in the appendix for reproducibility.

\item {\bf Crowdsourcing and research with human subjects}
    \item[] Question: For crowdsourcing experiments and research with human subjects, does the paper include the full text of instructions given to participants and screenshots, if applicable, as well as details about compensation (if any)? 
    \item[] Answer: \answerNA{}
    \item[] Justification: The work does not involve crowdsourcing, human participants, user studies, or human-subject data.

\item {\bf Institutional review board (IRB) approvals or equivalent for research with human subjects}
    \item[] Question: Does the paper describe potential risks incurred by study participants, whether such risks were disclosed to the subjects, and whether Institutional Review Board (IRB) approvals (or an equivalent approval/review based on the requirements of your country or institution) were obtained?
    \item[] Answer: \answerNA{}
    \item[] Justification: The work does not involve human subjects or crowdsourcing experiments, so IRB approval or equivalent review is not applicable.

\item {\bf Declaration of LLM usage}
    \item[] Question: Does the paper describe the usage of LLMs if it is an important, original, or non-standard component of the core methods in this research? Note that if the LLM is used only for writing, editing, or formatting purposes and does \emph{not} impact the core methodology, scientific rigor, or originality of the research, declaration is not required.
    \item[] Answer: \answerNA{}
    \item[] Justification: The core methodology, theory, algorithms, and experiments do not use LLMs as an important, original, or non-standard research component. Any ordinary writing, editing, or formatting assistance does not affect the scientific content of the work.

\end{enumerate}

\end{document}